\documentclass[10pt,twocolumn]{article}

\makeatletter
\renewcommand\section{\@startsection{section}{1}{\z@}%
  {1.2ex plus .2ex minus .2ex}%
  {0.6ex plus .2ex}%
  {\normalfont\large\bfseries}}
\renewcommand\subsection{\@startsection{subsection}{2}{\z@}%
  {0.9ex plus .2ex minus .2ex}%
  {0.45ex plus .2ex}%
  {\normalfont\normalsize\bfseries}}
\renewcommand\subsubsection{\@startsection{subsubsection}{3}{\z@}%
  {0.7ex plus .2ex minus .2ex}%
  {0.35ex plus .2ex}%
  {\normalfont\normalsize\itshape}}
\makeatother

\makeatletter
\renewcommand\maketitle{%
  \begingroup
  \renewcommand\thefootnote{\fnsymbol{footnote}}%
  \twocolumn[\vspace{-1.2em}%
    \begin{center}
      {\LARGE\bfseries \@title \par}
      \vspace{0.6em}
      {\large \@author \par}
      \vspace{0.4em}
    \end{center}
    \vspace{-0.6em}
  ]%
  \@thanks
  \setcounter{footnote}{0}%
  \endgroup
}
\makeatother

\newenvironment{iclrabstract}{%
  \par\small
  \noindent{\bfseries Abstract}\quad
}{\par\normalsize\vspace{0.5em}}

\usepackage{amssymb}
\usepackage{amsmath}
\usepackage{amsthm}
\usepackage{hyperref}
\usepackage{float}
\usepackage{cleveref}
\usepackage{tikz}
\usepackage{booktabs}
\usepackage{multirow}
\usetikzlibrary{decorations.pathreplacing,arrows.meta,calc}

\theoremstyle{definition}
\newtheorem{defn}{Definition}[]
\theoremstyle{plain}
\newtheorem{prop}{Proposition}[]

\title{Autonomous AI Agents for Option Hedging: Enhancing Financial Stability through Shortfall Aware Reinforcement Learning
}

\author{
Minxuan Hu\thanks{Cornell Ann S. Bowers College of Computing and Information Science, Cornell University, \texttt{mh2229@cornell.edu}},
Ziheng Chen\thanks{Department of Mathematics, University of Texas at Austin, \texttt{stokes615@utexas.edu}},
Jiayu Yi\thanks{School of Social Sciences, Nanyang Technological University, \texttt{sophiayi97@gmail.com}},
Wenxi Sun\thanks{Krieger School of Arts and Sciences, Johns Hopkins University, \texttt{wsun41@alumni.jh.edu}}
}

\begin{document}
\global\long\def\bC{\mathbb{C}}%

\global\long\def\bE{\mathbb{E}}%

\global\long\def\bF{\mathbb{F}}%

\global\long\def\bK{\mathbb{K}}%

\global\long\def\bN{\mathbb{N}}%

\global\long\def\bP{\mathbb{P}}%

\global\long\def\bQ{\mathbb{Q}}%

\global\long\def\bR{\mathbb{R}}%

\global\long\def\bT{\mathbb{T}}%

\global\long\def\bZ{\mathbb{Z}}%

\global\long\def\cA{\mathcal{A}}%

\global\long\def\cB{\mathcal{B}}%

\global\long\def\cC{\mathcal{C}}%

\global\long\def\cD{\mathcal{D}}%

\global\long\def\cE{\mathcal{E}}%

\global\long\def\cF{\mathcal{F}}%

\global\long\def\cG{\mathcal{G}}%

\global\long\def\cH{\mathcal{H}}%

\global\long\def\cI{\mathcal{I}}%

\global\long\def\cJ{\mathcal{J}}%

\global\long\def\cK{\mathcal{K}}%

\global\long\def\cL{\mathcal{L}}%

\global\long\def\cLp#1{\mathcal{L}^{#1}}%

\global\long\def\cLpp#1{\mathcal{L}_{+}^{#1}}%

\global\long\def\cLsimp{\mathcal{L}_{simp}^{0}}%

\global\long\def\cM{\mathcal{M}}%

\global\long\def\cN{\mathcal{N}}%

\global\long\def\cO{\mathcal{O}}%

\global\long\def\cP{\mathcal{P}}%

\global\long\def\cR{\mathcal{R}}%

\global\long\def\cS{\mathcal{S}}%

\global\long\def\cT{\mathcal{T}}%

\global\long\def\cU{\mathcal{U}}%

\global\long\def\cV{\mathcal{V}}%

\global\long\def\cW{\mathcal{W}}%

\global\long\def\cY{\mathcal{Y}}%

\global\long\def\cZ{\mathcal{Z}}%

\global\long\def\fq{\mathscr{\mathfrak{q}}}%

\global\long\def\fH{\mathscr{\mathfrak{H}}}%

\global\long\def\sD{\mathscr{D}}%

\global\long\def\sH{\mathscr{H}}%

\global\long\def\sK{\mathscr{K}}%

\global\long\def\sF{\mathscr{F}}%

\global\long\def\norm#1{\left\Vert #1\right\Vert }%

\global\long\def\np#1#2{\left\Vert #1\right\Vert _{#2}}%

\global\long\def\nlp#1#2{\left\Vert #1\right\Vert _{L^{#2}}}%

\global\long\def\abs#1{\left|#1\right|}%

\global\long\def\inv#1{#1^{-1}}%

\global\long\def\adjoint#1{#1^{*}}%

\global\long\def\annihilator#1{#1^{\circ}}%

\global\long\def\annihilatee#1{#1^{\perp}}%

\global\long\def\unaryop#1{#1\left(\cdot\right)}%

\global\long\def\binaryop#1{#1\left(\cdot,\cdot\right)}%

\global\long\def\comp#1#2{#1\circ#2}%

\global\long\def\converge#1{\overset{#1}{\joinrel\longrightarrow}}%

\global\long\def\define{\triangleq}%

\global\long\def\enum#1#2{\left\{  #1_{1},\dots,#1_{#2}\right\}  }%

\global\long\def\enumvec#1#2{\left(#1_{1},\dots,#1_{#2}\right)}%

\global\long\def\enuminf#1{\left\{  #1_{1},#1_{2}\dots\right\}  }%

\global\long\def\equivalent{\Longleftrightarrow}%

\global\long\def\substitute#1{\overset{#1}{\joinrel===}}%

\global\long\def\tensor{\otimes}%

\global\long\def\liminf#1{\underset{#1}{\operatorname{lim\,inf}}}%

\global\long\def\limsup#1{\underset{#1}{\operatorname{lim\,sup}}}%

\global\long\def\essinf#1{\underset{#1}{\operatorname{ess\,inf}}}%

\global\long\def\esssup#1{\underset{#1}{\operatorname{ess\,sup}}}%

\global\long\def\sgn{\operatorname{sgn}}%

\global\long\def\spanset{\operatorname{span}}%

\global\long\def\Null{\operatorname{Null}}%

\global\long\def\Range{\operatorname{Range}}%

\global\long\def\io{\operatorname{i.o.}}%

\global\long\def\ae{\operatorname{a.e.}}%

\global\long\def\as{\operatorname{a.s.}}%

\global\long\def\d#1{\operatorname{d}#1}%

\global\long\def\D#1{\operatorname{D}#1}%

\global\long\def\Db#1{\operatorname{D}\left[#1\right]}%

\global\long\def\cov{\operatorname{cov}}%

\global\long\def\supp{\operatorname{supp}}%

\maketitle





\begin{iclrabstract}
The deployment of autonomous AI agents in derivatives markets has widened a practical gap between static model calibration and realized hedging outcomes. We introduce two reinforcement learning frameworks, a novel Replication Learning of Option Pricing (RLOP) approach and an adaptive extension of Q-learner in Black-Scholes (QLBS), that prioritize shortfall probability and align learning objectives with downside sensitive hedging. Using listed SPY and XOP options, we evaluate models using realized path delta hedging outcome distributions, shortfall probability, and tail risk measures such as Expected Shortfall. Empirically, RLOP reduces shortfall frequency in most slices and shows the clearest tail-risk improvements in stress, while implied volatility fit often favors parametric models yet poorly predicts after-cost hedging performance. This friction-aware RL framework supports a practical approach to autonomous derivatives risk management as AI-augmented trading systems scale.

\end{iclrabstract}


\makeatother

\section{Introduction}

\subsection{Motivation}
The pervasive integration of artificial intelligence in recent years, specifically through the deployment of Large Language Models and the rise of autonomous financial agents, has fundamentally restructured the operational paradigms of financial services. Although commercial investment has facilitated significant progress in automated advisory, the application of such technologies to the risk management of complex derivatives remains a critical frontier. A persistent challenge in current financial applications driven by artificial intelligence is the methodological divergence between pricing model calibration and actual hedging performance. Traditional frameworks frequently rely on static diagnostics that fail to incorporate the operational realities of market frictions and equilibrium effects in incomplete markets. This research addresses this divergence by introducing a reinforcement learning framework that shifts the hedging objective from the traditional minimization of errors to the optimization of shortfall probability.

This research resolves the divergence between pricing calibration and execution \cite{gu2020empirical} by shifting the hedging objective to the optimization of shortfall probability. The study utilizes agents based on neural networks within modified QLBS and novel RLOP architectures to integrate risk aversion and market frictions into the decision process. Empirical results demonstrate that these reinforcement learning agents provide a systematic cost advantage over traditional Delta hedging by establishing a decision loop sensitive to costs. In the presence of substantial transaction costs, the modified QLBS agent prioritizes stability while the RLOP model mitigates margin pressure and liquidity demand. This resilience was validated during the 2020 crash, in which RLOP systematically reduced exposure to manage extreme stress. Distributional analysis confirms that this approach achieves material improvements in downside control by reducing extreme losses after costs, ensuring a robust equilibrium during operationally salient regimes of volatility.

This paper delivers three primary research contributions: (1) This paper extends the QLBS framework by embedding shortfall probability into the reward structure, resolving the decoupling between static calibration and execution. It demonstrates that IVRMSE-based diagnostics favor parametric models but fail to reflect hedging quality under frictions, necessitating a survival-centric strategy. (2) The novel RLOP model is introduced for superior tail-risk resilience. By prioritizing hedging success frequency over loss magnitude, RLOP provides material improvements in downside control. Distributional Expected Shortfall (ES) analysis confirms it significantly reduces extreme after-cost losses during operationally salient regimes, like the 2020 crash. (3) This study establishes a bidirectional selection framework supported by cost-risk maps and net CDF grids. This paper proves RL policies achieve a systematic cost advantage and turnover reduction. While QLBS acts as a cost-aware stabilizer, RLOP manages margin pressure, ensuring a robust equilibrium that avoids selection bias through a full-distribution view.

The remainder of this paper is structured as follows: Section 1.2 provides a comprehensive review of the relevant literature, Section 2 shows the reinforcement learning framework, Section 3 introduces the modified QLBS model and the novel RLOP model, Section 4 and 5 present their empirical analysis under varying risk and cost scenarios, and Section 6 concludes this paper.

\subsection{Literature review}
Option pricing and hedging remain core challenges in quantitative finance. 
 \cite{black1973pricing} provide a foundational option valuation model where perfect replication is achieved under frictionless markets with continuous-time trading, and subsequent research has extended BSM to adapt to more complex market realities (\cite{fan2022empirical,li2025analytic,golbabai2013superconvergence}). Nevertheless, real-world markets involve transaction costs and discrete-time trading. Classical options pricing models like Black-Scholes and stochastic volatility models are typically used for static calibration of the market’s implied volatility surface. However, as \cite{lassance2018comparison} have shown, these models do not necessarily produce better results when used to calculate hedging performance. This indicates that increasing sophistication in pricing models does not yield better hedging results as they represent two different objectives of pricing. Therefore, traditional diagnostic methods such as IVRMSE do not adequately reflect the actual hedging quality due to frictions in the real world.

Reinforcement learning has recently emerged as a new method of data-driven hedging, since it allows for direct optimization of realized performance instead of relying on model-specific assumptions. \cite{halperin2020qlbs,halperin2019qlbs} introduces a QLBS framework, which formulates hedging of options as a discrete-time reinforcement learning problem. \cite{buehler2019deep}further advanced RL-based hedging by introducing Deep Hedging. It is based on the use of neural networks to learn hedging policies so that transaction costs and market frictions can be included when optimizing under different risk measures (e.g., Expected Shortfall (ES)). Although these frameworks allow for the modeling of transaction costs, they still primarily focus on accuracy of replication with respect to the error-based reward functions. This can incentivize frequent trading and reduce overall cost effectiveness.

Although these reinforcement learning method have incorporated risk measures like Expected Shortfall, they have concentrated on the magnitude of the tail loss (expected value conditional on being in the worst $\alpha$-tail) rather than on the probability of achieving any loss at all. \cite{follmer2000efficient} showed that the optimal hedge under transaction costs should focus on the minimization of shortfall rather than replication error, trading off replication cost against downside protection. The magnitude of this distinction is made clear by the experience of 2020 and the extreme volatility caused by the COVID-19 pandemic. In particular, the crisis proved the need for ``survival'' strategies. This kind of strategies focus on reducing tail risk in order to create a business that will last and ensure that the business has a level of operational stability. Therefore, it is necessary to introduce shortfall-aware objectives to directly optimize the probability of hedging against future shortfalls into RL strategies.

\section{Replication Pricing and RL}
\label{sec:replication}

Replication-based option pricing builds a self-financing trading strategy whose terminal portfolio value replicates the option payoff.
In the classical setting of \cite{black1973pricing}, replication is achieved by continuously (or discretely) rebalancing a hedging portfolio.
Given a price process ${S_t}$ adapted to a filtration ${\cF_t}$, the portfolio holds $u_t$ units of the underlying asset and a risk-free account $B_t$, with portfolio value
$
\Pi_t := u_t S_t + B_t .
$
The self-financing constraint stipulates that rebalancing is funded entirely by the existing portfolio, so no external cash is added or withdrawn, which leads to
{\scriptsize%
\begin{equation}
u_{t}S_{t+1}+e^{r\Delta t}B_{t}=u_{t+1}S_{t+1}+B_{t+1} + \text{TC}(u_{t+1}-u_t, S_{t+1}) \label{eq:self-financing}
\end{equation}
}where $r$ is the risk-free rate and $\text{TC}(\cdot)$ denotes transaction costs.
\Cref{eq:self-financing} characterizes the portfolio dynamics and, in particular, the capital needed to maintain trading over time.
Throughout, we assume the underlying price follows geometric Brownian motion,
$ \d S_t=\mu ,S_t \d{t} + \sigma ,S_t \d{W_t} $,
with drift $\mu$, volatility $\sigma$, and Wiener process $W_t$.


Reinforcement learning (RL) models sequential decision-making problems through a Markov decision process (MDP).
Following \cite{sutton2018reinforcement,bertsekas2019reinforcement}, an MDP is characterized by a state space $\mathcal S$, an action space $\mathcal A$, a transition kernel $p(\cdot \mid s,a)$, and a reward function $R(s,a)$.
A policy $\pi$ maps the current state to an action; in our context, $\pi$ serves as a hedging rule that updates the trading position over time.
Inspired by the Girsanov transform \cite{liptser2013statistics}, we formulate option replication as an MDP with state $(t,X_t)$, where the normalized price process is
$
X_t := -\Bigl(\mu - \tfrac{\sigma^2}{2}\Bigr)t + \log S_t .
$
The action $a_t$ denotes the hedge position chosen based on the normalized input $X_t$.
The transition kernel $p$ (and thus the realized rewards) is determined by the particular model specification (e.g., QLBS versus RLOP, as discussed below).
To avoid ambiguity across representations, we write $a_t$ for the hedge expressed as a function of $X_t$, while $u_t$ denotes the corresponding hedge position on the original price scale $S_t$.


As shown in \cite{halperin2020qlbs}, the Black-Scholes price admits a discrete-time characterization as the conditional expectation of the replicating portfolio value $\Pi_t$.
This observation motivates a control problem in which the fair option price is given by the maximized state value function
$
\tilde{V}_{t}^{\pi}\left(X_{t}\right)=\bE^{\pi}_t\left[-\Pi_{t}\left(X_{t}\right)-\lambda\sum_{\tau=t}^{T}e^{-r\left(\tau-t\right)}\text{Var}_t\left[\Pi_{\tau}\left(X_{\tau}\right)\right]\right]
$
where $\lambda$ denotes the risk-aversion parameter and $\bE_t$ is the conditional expectation with respect to $\cF_t$.
Within this setup, the option price is equal to the negative of the optimal value function.


\cite{halperin2020qlbs} obtains a closed-form maximizer of
$\tilde{V}_{t}^{\pi}$
by leveraging its quadratic mean-variance structure, which is tractable when the correlation between $\Pi_{t+1}$ and $\Delta S$ is available.
Nonetheless, this method is tightly tied to the stylized assumptions and does not readily extend to more general settings.
For a given payoff function $h$, enforcing the terminal condition $\Pi_T = h(S_T)$ under the self-financing constraint (\cref{eq:self-financing}) typically yields a portfolio process $\Pi_t$ that is non-adapted, which makes the direct use of standard RL algorithms nontrivial.

\section{Two RL Formulations for Option Pricing and Hedging}
\label{method}


To address the limitations of the original QLBS formulation, two complementary approaches can be pursued.
The first extends the QLBS framework by redefining the value function $V_t^\pi$ so that it becomes an ${\cF_t}$-adapted process.
The second takes a different viewpoint by constructing an adaptive portfolio value process and specifying the reward via the terminal hedging tracking error.
Both approaches incorporate transaction costs explicitly and remain compatible with both value-based and policy-based reinforcement learning algorithms.

\subsection{Adaptive-QLBS: Backward Value-Based RL}

\begin{defn}
\label{def:qlbs}
Let $d_T(t):=\left(1-\frac{t}{T}\right)$ and $\gamma:=e^{-r \Delta t}$.
The value function of the adaptive-QLBS method reads
{\scriptsize%
\begin{align}
V_{t}^{\pi}\left(X_{t}\right)&:=\bE_{t}^{\pi}\left[-d_T(t)\Pi_{t}\left(X_{t}\right)-\lambda\sum_{\tau=t}^{T}\gamma^{\tau-t}\sqrt{\text{Var}\left[\Pi_{\tau}\left(X_{\tau}\right)\right]}\right],\label{eq:qlbs-modified-V}
\end{align}
}with $R_{t+1}\left(X_{t},a_{t}\right):=V_t^\pi(X_t)-\bE_{t}^{\pi} V_{t+1}^\pi(X_{t+1})$ as the reward function where $X_{t+1}$ is implicitly determined by $X_t,a_t$ and the self-financing condition \cref{eq:self-financing}
\end{defn}


Our modifications are twofold:
we introduce a discounting factor $d_T(t)$ that decreases from $1$ at $t=0$ to $0$ at $t=T$, thereby smoothing the influence of the terminal payoff on the portfolio term,
and we replace variance terms with their square roots to obtain a dimensionless, numerically more stable value estimate.
With transaction costs, the portfolio value process $\Pi_t$ is computed backward via the self-financing relation in \cref{eq:self-financing}.
A schematic overview of the adaptive-QLBS model is shown in \cref{fig:qlbs-illustration}.

\begin{figure*}[htbp]
    \centering
    \hfill
    \begin{minipage}{0.35\textwidth}
        \centering
        \begin{tikzpicture}[x=0.8cm,y=0.8cm, line cap=round, line join=round]
\begin{scope}[rotate=90]
  \coordinate (O)   at (2,0); 
  \coordinate (t0)  at (4,0); \coordinate (t0a)  at (4,2.5);
  \coordinate (T)   at (7.5,0);
  \coordinate (tau) at (7.65,0);

  \draw[thick,dotted] ($(O)+(0, -1)$) -- ($(t0)+(0, -1)$);
  \draw[thick,-stealth] ($(t0)+(0, -1)$) --($(T)+(0, -1)$);

  \node[right] at ($(O)+(0, -1)$) {$0$};
  \node[right]      at ($(t0)+(0, -1)$) {$t$};
  \node[right]      at ($(T)+(0, -1)$) {$T$};

  \fill ($(t0)+(0, -1)$) circle (2.2pt);


  \node[above] at ($(6.85,2.05)$) {$S_T$};
  \draw[color=gray] (t0a) --
  ++(0.055,0.117) -- ++(0.055,0.074) -- ++(0.055,0.149) -- ++(0.055,0.151) -- ++(0.055,0.152) -- ++(0.055,0.084) -- ++(0.055,0.018) -- ++(0.055,0.056) -- ++(0.055,0.105) -- ++(0.055,0.045) -- ++(0.055,0.083) -- ++(0.055,0.094) -- ++(0.055,0.085) -- ++(0.055,-0.047) -- ++(0.055,0.020) -- ++(0.055,0.103) -- ++(0.055,0.035) -- ++(0.055,0.006) -- ++(0.055,0.015) -- ++(0.055,0.013) -- ++(0.055,0.089) -- ++(0.055,-0.042) -- ++(0.055,0.073) -- ++(0.055,0.009) -- ++(0.055,0.006) -- ++(0.055,0.007) -- ++(0.055,0.054) -- ++(0.055,0.057) -- ++(0.055,0.050) -- ++(0.055,-0.035) -- ++(0.055,0.075) -- ++(0.055,0.051) -- ++(0.055,0.046) -- ++(0.055,-0.031) -- ++(0.055,0.088) -- ++(0.055,0.047) -- ++(0.055,0.019) -- ++(0.055,0.099) -- ++(0.055,-0.010) -- ++(0.055,0.073) -- ++(0.055,0.006) -- ++(0.055,-0.072) -- ++(0.055,0.027) -- ++(0.055,0.074) -- ++(0.055,-0.008) -- ++(0.055,-0.019) -- ++(0.055,-0.064) -- ++(0.055,0.020) -- ++(0.055,-0.035) -- ++(0.055,-0.007) -- ++(0.055,0.022) -- ++(0.055,-0.052) -- ++(0.055,0.058) -- ++(0.055,-0.052) -- ++(0.055,-0.032) -- ++(0.055,0.067) -- ++(0.055,-0.109) -- ++(0.055,0.025) -- ++(0.055,-0.021) -- ++(0.055,0.065) -- ++(0.055,0.026) -- ++(0.055,0.067) -- ++(0.055,-0.016) ;
  \draw[color=gray] (t0a) --
  ++(0.055,0.049) -- ++(0.055,0.082) -- ++(0.055,0.053) -- ++(0.055,0.020) -- ++(0.055,0.061) -- ++(0.055,0.102) -- ++(0.055,0.079) -- ++(0.055,-0.002) -- ++(0.055,-0.029) -- ++(0.055,0.003) -- ++(0.055,0.034) -- ++(0.055,-0.079) -- ++(0.055,0.024) -- ++(0.055,-0.025) -- ++(0.055,0.000) -- ++(0.055,0.009) -- ++(0.055,0.019) -- ++(0.055,0.052) -- ++(0.055,0.079) -- ++(0.055,0.020) -- ++(0.055,0.089) -- ++(0.055,-0.011) -- ++(0.055,0.037) -- ++(0.055,0.061) -- ++(0.055,0.019) -- ++(0.055,-0.021) -- ++(0.055,-0.028) -- ++(0.055,-0.005) -- ++(0.055,0.027) -- ++(0.055,-0.032) -- ++(0.055,0.007) -- ++(0.055,0.009) -- ++(0.055,0.042) -- ++(0.055,0.024) -- ++(0.055,0.029) -- ++(0.055,-0.020) -- ++(0.055,0.006) -- ++(0.055,0.049) -- ++(0.055,0.079) -- ++(0.055,-0.054) -- ++(0.055,0.079) -- ++(0.055,0.067) -- ++(0.055,0.036) -- ++(0.055,0.010) -- ++(0.055,-0.018) -- ++(0.055,0.067) -- ++(0.055,0.086) -- ++(0.055,0.074) -- ++(0.055,0.047) -- ++(0.055,-0.000) -- ++(0.055,-0.074) -- ++(0.055,-0.013) -- ++(0.055,0.019) -- ++(0.055,-0.074) -- ++(0.055,0.009) -- ++(0.055,0.011) -- ++(0.055,0.022) -- ++(0.055,-0.067) -- ++(0.055,-0.039) -- ++(0.055,-0.026) -- ++(0.055,0.026) -- ++(0.055,0.067) -- ++(0.055,-0.016) ;
  \draw[thick] (t0a) --
  ++(0.055,-0.009) -- ++(0.055,0.007) -- ++(0.055,-0.004) -- ++(0.055,-0.056) -- ++(0.055,0.029) -- ++(0.055,-0.023) -- ++(0.055,0.013) -- ++(0.055,-0.030) -- ++(0.055,0.054) -- ++(0.055,-0.011) -- ++(0.055,0.011) -- ++(0.055,-0.004) -- ++(0.055,0.023) -- ++(0.055,-0.002) -- ++(0.055,0.004) -- ++(0.055,0.038) -- ++(0.055,0.009) -- ++(0.055,0.021) -- ++(0.055,-0.024) -- ++(0.055,0.021) -- ++(0.055,-0.030) -- ++(0.055,0.020) -- ++(0.055,-0.045) -- ++(0.055,-0.050) -- ++(0.055,-0.009) -- ++(0.055,0.045) -- ++(0.055,0.008) -- ++(0.055,-0.015) -- ++(0.055,-0.004) -- ++(0.055,-0.065) -- ++(0.055,0.017) -- ++(0.055,-0.029) -- ++(0.055,0.022) -- ++(0.055,0.057) -- ++(0.055,0.005) -- ++(0.055,-0.007) -- ++(0.055,0.029) -- ++(0.055,0.010) -- ++(0.055,-0.012) -- ++(0.055,-0.031) -- ++(0.055,-0.043) -- ++(0.055,0.048) -- ++(0.055,0.019) -- ++(0.055,0.061) -- ++(0.055,-0.011) -- ++(0.055,-0.032) -- ++(0.055,0.030) -- ++(0.055,-0.026) -- ++(0.055,-0.058) -- ++(0.055,0.049) -- ++(0.055,-0.026) -- ++(0.055,0.014) -- ++(0.055,-0.019) -- ++(0.055,0.037) -- ++(0.055,-0.025) -- ++(0.055,0.009) -- ++(0.055,0.005) -- ++(0.055,-0.005) -- ++(0.055,0.044) -- ++(0.055,0.026) -- ++(0.055,0.026) -- ++(0.055,0.067);
  \draw[color=gray] (t0a) --
  ++(0.055,-0.127) -- ++(0.055,0.040) -- ++(0.055,-0.173) -- ++(0.055,-0.052) -- ++(0.055,-0.075) -- ++(0.055,-0.066) -- ++(0.055,-0.064) -- ++(0.055,-0.034) -- ++(0.055,0.024) -- ++(0.055,0.034) -- ++(0.055,-0.066) -- ++(0.055,-0.038) -- ++(0.055,-0.038) -- ++(0.055,-0.013) -- ++(0.055,-0.014) -- ++(0.055,0.023) -- ++(0.055,-0.013) -- ++(0.055,0.066) -- ++(0.055,-0.027) -- ++(0.055,-0.012) -- ++(0.055,-0.045) -- ++(0.055,-0.074) -- ++(0.055,-0.084) -- ++(0.055,0.018) -- ++(0.055,-0.079) -- ++(0.055,-0.042) -- ++(0.055,-0.057) -- ++(0.055,-0.009) -- ++(0.055,0.005) -- ++(0.055,-0.062) -- ++(0.055,-0.036) -- ++(0.055,0.014) -- ++(0.055,-0.038) -- ++(0.055,-0.035) -- ++(0.055,-0.006) -- ++(0.055,-0.051) -- ++(0.055,0.176) -- ++(0.055,0.042) -- ++(0.055,-0.060) -- ++(0.055,0.025) -- ++(0.055,0.002) -- ++(0.055,-0.044) -- ++(0.055,0.036) -- ++(0.055,0.027) -- ++(0.055,-0.097) -- ++(0.055,0.029) -- ++(0.055,-0.003) -- ++(0.055,-0.066) -- ++(0.055,-0.080) -- ++(0.055,-0.048) -- ++(0.055,-0.035) -- ++(0.055,0.035) -- ++(0.055,0.086) -- ++(0.055,0.015) -- ++(0.055,0.064) -- ++(0.055,0.049) -- ++(0.055,0.046) -- ++(0.055,-0.136) -- ++(0.055,0.072) -- ++(0.055,0.016)  -- ++(0.055,0.026) -- ++(0.055,0.067);
  \draw[color=gray] (t0a) --
  ++(0.055,-0.032) -- ++(0.055,-0.053) -- ++(0.055,-0.214) -- ++(0.055,-0.108) -- ++(0.055,-0.126) -- ++(0.055,-0.151) -- ++(0.055,-0.081) -- ++(0.055,-0.168) -- ++(0.055,-0.102) -- ++(0.055,-0.070) -- ++(0.055,-0.016) -- ++(0.055,-0.028) -- ++(0.055,-0.087) -- ++(0.055,-0.023) -- ++(0.055,-0.027) -- ++(0.055,-0.027) -- ++(0.055,-0.137) -- ++(0.055,0.060) -- ++(0.055,-0.097) -- ++(0.055,-0.146) -- ++(0.055,-0.040) -- ++(0.055,-0.066) -- ++(0.055,-0.011) -- ++(0.055,0.046) -- ++(0.055,-0.061) -- ++(0.055,-0.062) -- ++(0.055,0.002) -- ++(0.055,0.001) -- ++(0.055,0.013) -- ++(0.055,0.071) -- ++(0.055,0.059) -- ++(0.055,0.001) -- ++(0.055,-0.013) -- ++(0.055,0.016) -- ++(0.055,0.022) -- ++(0.055,0.031) -- ++(0.055,0.022) -- ++(0.055,-0.054) -- ++(0.055,0.021) -- ++(0.055,-0.017) -- ++(0.055,-0.012) -- ++(0.055,-0.009) -- ++(0.055,-0.011) -- ++(0.055,0.031) -- ++(0.055,0.056) -- ++(0.055,-0.114) -- ++(0.055,-0.057) -- ++(0.055,0.092) -- ++(0.055,-0.048) -- ++(0.055,0.058) -- ++(0.055,-0.019) -- ++(0.055,0.110) -- ++(0.055,0.034) -- ++(0.055,-0.016) -- ++(0.055,-0.015) -- ++(0.055,-0.041) -- ++(0.055,-0.065) -- ++(0.055,-0.102) -- ++(0.055,-0.056) -- ++(0.055,-0.058) -- ++(0.055,-0.028) -- ++(0.055,-0.034) ;
  \draw[thick,dotted] (t0a) --
  ++(-0.050,0.049) -- ++(-0.050,-0.165) -- ++(-0.050,0.030) -- ++(-0.050,0.037) -- ++(-0.050,0.257) -- ++(-0.050,-0.076) -- ++(-0.050,-0.015) -- ++(-0.050,0.020) -- ++(-0.050,-0.223) -- ++(-0.050,-0.071) -- ++(-0.050,0.079) -- ++(-0.050,-0.094) -- ++(-0.050,-0.127) -- ++(-0.050,0.003) -- ++(-0.050,0.103) -- ++(-0.050,-0.191) -- ++(-0.050,-0.141) -- ++(-0.050,-0.083) -- ++(-0.050,-0.069) -- ++(-0.050,0.002) -- ++(-0.050,-0.159) -- ++(-0.050,0.109) -- ++(-0.050,0.158) -- ++(-0.050,0.006) -- ++(-0.050,-0.092) -- ++(-0.050,0.014);

  \draw[-{Stealth[length=2mm,width=1.5mm]}] ($(t0)+(0,1.5)$) -- ($(t0)+(0,2)$);

  \draw[dashed,thick] ($(tau)+(0,0.3)$) -- ($(tau)+(0,4.8)$);


  \node[align=center] at ($(t0a)+(1.3,2.7)$)
    {$\text{Var}\,\Pi_t$ \\under\\ \cref{eq:self-financing}};

  \node[align=center,above] at ($(t0a)+(-0.7,-2.05)$) {hedge $u_t$ \\ state $(t,S_t)$};

  \node[align=right] at ($(tau)+(0.5,2.5)$)
    {terminal condition $h$};

\end{scope}
\end{tikzpicture}
        \caption{The adaptive-QLBS method takes a backward, value-based approach.}
        \label{fig:qlbs-illustration}
    \end{minipage}
    \hfill
    \begin{minipage}{0.55\textwidth}
        \centering
        \begin{tikzpicture}[x=0.8cm,y=0.8cm,>=Stealth]
\begin{scope}[yscale=-1]
  \tikzset{
    tl/.style={line width=0.9pt},
    pf/.style={line width=0.9pt},
    tick/.style={line width=0.9pt},
    darr/.style={dashed, line width=0.8pt, ->},
    dot/.style={circle, fill=black, inner sep=1.3pt},
  }

  \def\xO{0}
  \def\xOne{1}
  \def\xTwo{2}
  \def\xTmid{6}  
  \def\xT{9}     

  \draw[tl] (\xO,0) -- (\xT,0);

  \foreach \x/\lab in {\xO/0,\xOne/1,\xTwo/2,\xTmid/$t$,\xT/$T$}{
    \draw[tick] (\x,0.12) -- (\x,-0.12);
    \node[below] at (\x,0.18) {\lab};
  }
  \node[below] at (3.9,0.2) {$\cdots$};
  \node[below] at (7.5,0.2) {$\cdots$};

  \node[anchor=east] at (-0.25,-0.3) {portfolio};
  \node[anchor=east,color=gray] at (-0.25,-1.1) {\#1};
  \node[anchor=east,color=gray] at (-0.25,-2.00) {\#2};
  \node[anchor=east] at (-0.25,-3.30) {\#$t$};
  \node[anchor=east,color=gray] at (-0.25,-4.6) {\#$T$};

  \node[dot,color=gray] at (\xO,-1.10) {};
  \draw[pf,color=gray] (\xO,-1.10) -- (\xOne,-1.10);
  \node[above,color=gray] at (\xO+0.5,-1.10) {$u^{(1)}_{0}$};

  \draw[darr,color=gray] (\xOne,-1.10) -- (\xOne,-0.2);
  \node[right,color=gray] at (\xOne, -0.7) {$R_{1}$};

  \node[dot,color=gray] at (\xO,-2.00) {};
  \draw[pf,color=gray] (\xO,-2.00) -- (\xTwo,-2.00);
  \draw[tick,color=gray] (\xOne,-2.00+0.12) -- (\xOne,-2.00-0.12);
  \node[above,color=gray] at (\xO+0.5,-2.00) {$u^{(2)}_{0}$};
  \node[above,color=gray] at (\xOne+0.50,-2.00) {$u^{(2)}_{1}$};

  \draw[darr,color=gray] (\xTwo,-2.00) -- (\xTwo,-0.2);
  \node[right,color=gray] at (\xTwo, -1.10) {$R_{2}$};

  \node[color=gray] at (-0.55,-2.70) {$\vdots$};

  \node[dot] at (\xO,-3.30) {};
  \draw[pf] (\xO,-3.30) -- (\xTmid,-3.30);
  \draw[tick] (\xOne,-3.30+0.12) -- (\xOne,-3.30-0.12);
  \draw[tick] (\xOne+1.0,-3.30+0.12) -- (\xOne+1.0,-3.30-0.12);
  \draw[tick] (\xTmid-1.0,-3.30+0.12) -- (\xTmid-1.0,-3.30-0.12);
  \node[above] at (\xO+0.5,-3.30) {$u^{(t)}_{0}$};
  \node[above] at (\xOne+0.50,-3.30) {$u^{(t)}_{1}$};
  \node[above] at (\xTmid-0.5,-3.30) {$u^{(t)}_{t-1}$};
  \node[above] at (3.0,-3.30) {$\cdots$};

  \draw[darr] (\xTmid,-3.30) -- (\xTmid,-0.2);
  \node[right] at (\xTmid, -1.65) {$R_{t}$};

  \node[color=gray] at (-0.55,-4.00) {$\vdots$};

  \node[dot,color=gray] at (\xO,-4.6) {};
  \draw[pf,color=gray] (\xO,-4.6) -- (\xT,-4.6);
  \draw[tick,color=gray] (\xOne,-4.6+0.12) -- (\xOne,-4.6-0.12);
  \draw[tick,color=gray] (\xOne+1,-4.6+0.12) -- (\xOne+1,-4.6-0.12);
  \draw[tick,color=gray] (\xT-1,-4.6+0.12) -- (\xT-1,-4.6-0.12);
  \node[above,color=gray] at (\xO+0.5,-4.6) {$u^{(T)}_{0}$};
  \node[above,color=gray] at (\xOne+0.50,-4.6) {$u^{(T)}_{1}$};
  \node[above,color=gray] at (\xT-0.5,-4.6) {$u^{(T)}_{T-1}$};
  \node[above,color=gray] at (5.0,-4.6) {$\cdots$};

  \draw[darr,color=gray] (\xT,-4.6) -- (\xT,-0.2);
  \node[right,color=gray] at (\xT, -2.3) {$R_{T}$};
\end{scope}
\end{tikzpicture}
        \caption{The RLOP method takes a forward, replication-based approach.}
        \label{fig:rlop-illustration}
    \end{minipage}
    \hfill

\end{figure*}


The proposition below clarifies why the option price increases with the key parameters: greater risk aversion $\lambda$ and higher transaction frictions both lead to a higher price. We first establish monotonicity of the value function under any fixed policy, and then derive monotonicity of the optimal value by taking the maximizer.

\begin{prop}
For sufficiently large $\epsilon$ that appears in the linear transaction cost assumption $\text{TC}(\Delta u, S)=\epsilon|\Delta u| \,S$, the option price $C(S_0) := -\max_{\pi\in\mathbf{\Pi}} V_0^\pi$ is monotonically increasing in both $\lambda$ and $\epsilon$.
\label{prop:qlbs-monotone}
\end{prop}
\begin{proof}
Given any policy $\pi$, the portfolio value at time $t$ may be written as
\begin{align}
\Pi_t &= u_t  S_t + \epsilon \mathfrak{S}_t + \gamma^{T-t} \mathfrak{E}_t, \label{eq:qlbs-proof-port-value} \\
\mathfrak{S}_t :&= \sum_{j=0}^{T-t-1} \gamma^{j+1} |\Delta u_{t+j}| S_{t+j+1}, \nonumber \\
\mathfrak{E}_t :&= h(S_T) - u_{T-1}S_T \nonumber
\end{align}
under the self-financing condition
(eq. 1)
.
This expression is affine in $\epsilon$ and independent of $\lambda$.
Hence, the monotonicity of $V_t^\pi$ with respect to  $\lambda$ follows immediately from
(eq. 2)
, since the discounted risk term is always non-positive.
For $\epsilon\gg1$, the monotonicity in $\epsilon$ follows from the fact that $\text{Var}_t\Pi_{t}\left(X_{t}\right)$ is a quadratic function of $\epsilon$ with non-negative leading coefficient.

To establish that $C(S_0,\lambda)$ is increasing in $\lambda$, let $\pi(\lambda)$ denote the maximizer of $V_t^\pi$.
For $\lambda ' > \lambda$, we have
$V_t^{\pi(\lambda )}(S_t;\lambda) \ge V_t^{\pi(\lambda')}(S_t;\lambda)$
because $\pi(\lambda)$  maximizes the value function at level $\lambda$.
From the argument above, $V_t^{\pi(\lambda')}(S_t;\lambda') \le V_t^{\pi(\lambda')}(S_t;\lambda)$.
Combining these inequalities yields
\begin{align*}
C(S_0,\lambda') &= -V_0^{\pi(\lambda')} (S_0;\lambda') \ge -V_0^{\pi(\lambda')} (S_0;\lambda) \\
&\ge -V_0^{\pi(\lambda)} (S_0;\lambda) = C(S_0,\lambda) .
\end{align*}
Thus, $C$ is monotone in $\lambda$.
The argument for its monotonicity in $\epsilon$ proceeds similarly.


\end{proof}


\subsection{RLOP: Forward Replication Learning}

We propose Replication Learning of Option Pricing (RLOP), which adopts a forward-looking formulation.
The agent trades a self-financing portfolio and receives rewards based on how closely its terminal wealth matches the option payoff.
In contrast to Deep Hedging (\cite{buehler2019deep}), which may implicitly allow a speculative component (\cite{franccois2025difference}), RLOP’s shortfall-probability objective promotes capital preservation and downside-sensitive hedging.
Motivated by reward shaping (\cite{sutton2018reinforcement,devlin2011theoretical}), we stack an ensemble of maturities: along a sample path ${S_t}$ over horizon $T$, the agent jointly manages portfolios $\Pi_t^{(i)}$ for expiries $i=1,\dots,T$, selecting hedge positions $u_t^{(i)}$ for all $t<i$.
This yields intermediate learning signals and allows the policy to learn on shorter horizons before extending to the full maturity.
A formal definition of the RLOP problem is given below.

\begin{defn}
The transitional probability density from $X_{t}=\left(t,S_{t}\right)$ to $X_{t+1}=\left(t+1,S_{t+1}\right)$ is defined as
$
    p\left(X_t,R_{t+1}\big|X_{t+1},u_{t}^{\left(i\right)}\right)=  \rho\left(S_{t},S_{t+1}\right) \boldsymbol{1}_{t<i}
$
where $\rho$ is specified by the discrete-time geometric Brownian motion dynamics.
The corresponding reward is defined as $R_i = H!\left(h(S_i), \Pi_i^{(i)}\right)$, where the penalty function $H$ quantifies the replication accuracy of the portfolio value $\Pi_i^{(i)}$ relative to the option payoff $h(S_i)$.
\end{defn}

In practice, we take $H(x,y)=-|x-y|$ or its squared variant, which directly penalizes terminal replication error.
The structure of the RLOP approach is illustrated in \cref{fig:rlop-illustration}.

\subsection{Neural Policy Training}




We parametrize the hedging policy in both QLBS and RLOP using neural networks and train it within a simulated environment.
The environment simulates geometric-Brownian price paths with parameters $(r,\mu,\sigma,T)$.
At each time $t$, the agent observes the normalized state $(t,X_t)$, outputs a hedge position $a_t$, and receives rewards computed from \cref{def:qlbs}; performance is evaluated via Monte Carlo rollouts.

We represent the policy as a Gaussian $\pi=\mathcal N(\mu_\pi,\sigma_\pi)$, where $\mu_\pi$ and $\sigma_\pi$ are produced by a shared ResNet-style architecture (\cite{he2016deep}).
A separate value network with the same architecture provides a learned baseline to reduce gradient variance.
Training follows REINFORCE with a baseline (\cite{williams1992simple,sutton2018reinforcement}) and is optimized using Adam with learning rate $10^{-4}$.

Finally, we numerically validate the monotonicity in \cref{prop:qlbs-monotone}.
\Cref{fig:sigma-qlbs-rlop,fig:qq-qlbs} indicate that the learned prices move consistently with $\sigma$, $\mu$, $\lambda$, and $\epsilon$, aligning with the theoretical predictions and reproducing implied-volatility skew patterns observed empirically.

\begin{figure*}[htb]
  \centering
  \hfill
  \begin{minipage}[b]{0.4\textwidth}
    \centering
    \includegraphics[width=\textwidth]{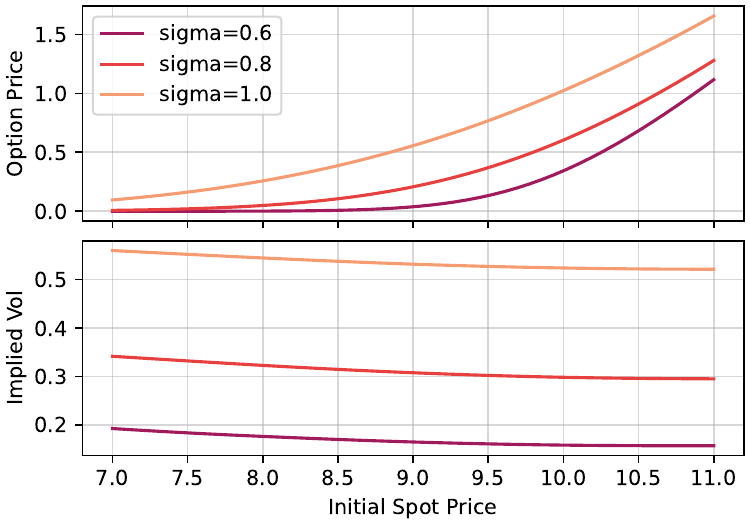}
  \end{minipage}
  \hfill
  \begin{minipage}[b]{0.4\textwidth}
    \centering
    \includegraphics[width=\textwidth]{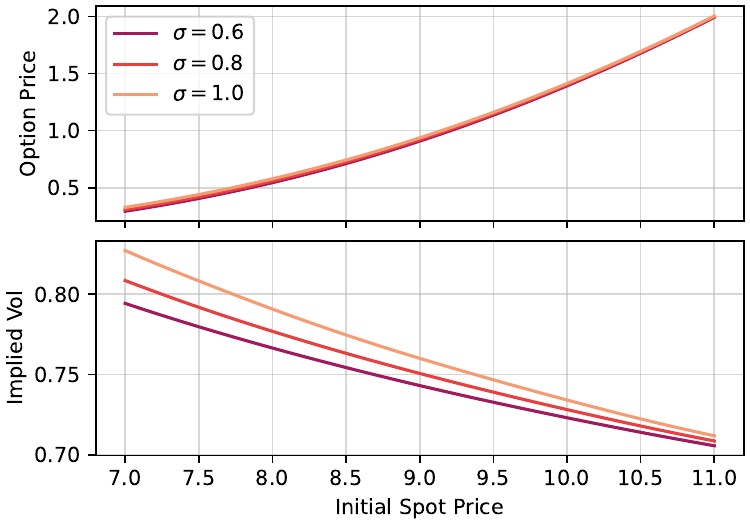}
  \end{minipage}
  \hfill
\caption{Price under RLOP model (left) and Adaptive-QLBS model (right) given different parameters of volatility. The common setup uses maturity $T=2$ months, strike $K=1$, interest rate $r=4\%$.}
\label{fig:sigma-qlbs-rlop}
\end{figure*}

\begin{figure*}[htb]
  \centering
  \hfill
  \begin{minipage}[b]{0.25\textwidth}
    \centering
    \includegraphics[width=\textwidth]{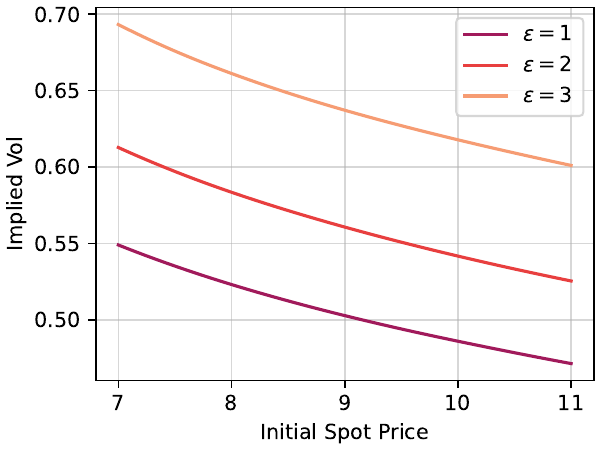}
  \end{minipage}
  \hfill
  \begin{minipage}[b]{0.25\textwidth}
    \centering
    \includegraphics[width=\textwidth]{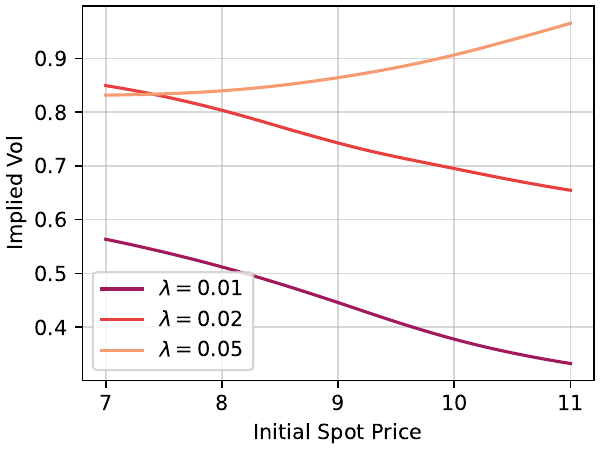}
  \end{minipage}
  \hfill
  \begin{minipage}[b]{0.25\textwidth}
    \centering
    \includegraphics[width=\textwidth]{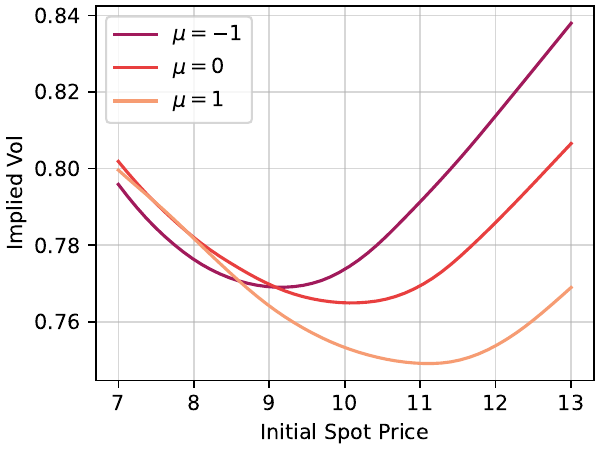}
  \end{minipage}
  \hfill
\caption{Price under Adaptive-QLBS model given different levels of hyperparameters: friction $\epsilon$ (left), risk aversion intensity $\lambda$ (middle), and drift $\mu$ (right).}
\label{fig:qq-qlbs}
\end{figure*}

\section{Empirical Results on Market Data}
\label{sec:empirical_market}

We evaluate models by their realized-path $\Delta$-hedging performance on listed option markets under discrete rebalancing and proportional transaction costs \cite{leland1985option}.
The question is operational: \emph{given a model’s deltas, what distribution of after-cost hedging outcomes is delivered along the subsequently realized underlying path?}
Accordingly, we emphasize distributional evidence for the net hedging outcome and downside risk, complemented by a compact risk--cost map that separates replication quality from execution costs.
Same-day implied-volatility fit is reported only as a diagnostic for pricing misspecification, not as the primary criterion for hedging quality.

We study European-style call contracts on SPY (S\&P 500 ETF) and XOP (energy-sector ETF) across two non-overlapping quarters chosen to represent distinct regimes:
2020Q1 (COVID dislocation) and 2025Q2 (calmer conditions).
We compare two RL-based methods (QLBS and RLOP) against three parametric benchmarks:
Black--Scholes/Black--76 \cite{black1973pricing,black1976commodity} (BS), Merton’s jump-diffusion \cite{merton1976discontinuous} (JD), and Heston stochastic volatility \cite{heston1993closed} (SV).

\subsection{Market Data and Experimental Slices}
\label{subsec:market_slices}

\paragraph{Daily option slices and contract universe.}
We use daily option snapshots for SPY and XOP and construct European-equivalent market call prices $C^{\mathrm{mkt}}(K,\tau)$ using the dataset-provided forwards and discount factors.
We retain contracts with $3$ to $70$ calendar days to maturity.
Each trading day defines one cross-sectional slice used for calibration/fitting, and all reported statistics aggregate over days with equal-day weighting so that
each date contributes one observation.

\paragraph{Maturity buckets and moneyness targets.}
To standardize comparisons across dates, we group maturities into three buckets centered at $\tau\in\{14,28,56\}$ days.
Within each bucket, moneyness is defined as $K/F$, where $F$ is the forward corresponding to maturity $\tau$.
In the main text we focus on the $\tau=28$d bucket and report two benchmark targets:
ATM ($K/F=1$) and mildly out-of-the-money ($K/F=1.03$).
The $\tau=14$ and $\tau=56$ buckets are reported in the appendix to document horizon robustness without overloading the main narrative.

\paragraph{Day-by-day calibration and delta generation.}
On each trading day and within each maturity bucket, the parametric models (BS, JD, SV) are calibrated to the same-day option cross-section by minimizing squared
pricing errors.
QLBS and RLOP are fit on the same day--bucket slice following the procedures in earlier sections.
The fitted models then output model-implied deltas, which are evaluated under a common realized-path hedging protocol with transaction costs (Sec.~\ref{subsubsec:realized_path_hedging}).

\subsection{Dynamic Hedging Performance}
\label{subsec:dynamic_hedging}

\subsubsection{Realized-Path Hedging under Transaction Costs}
\label{subsubsec:realized_path_hedging}

Dynamic evaluation is performed via discrete-time $\Delta$-hedging of a short call over the remaining life of the contract.
On each trade date $t_0$, a model is fit using the same-day option cross-section and supplies a delta rule
$\Delta_t=\Delta^{(m)}(S_t,t; \theta_{t_0})$ for $t\in[t_0,T]$, which is then applied along the subsequently realized underlying path.
(Here $\Delta_t$ denotes the hedge position, corresponding to $u_t$ in Sec.~\ref{sec:replication})
Rebalancing occurs once per trading day. All models are evaluated under the same rebalancing schedule and transaction-cost specification.

\paragraph{Self-financing portfolio with proportional costs.}
Let $S_t$ denote the underlying price, and consider a self-financing hedging portfolio holding $\Delta_t$ shares and a cash account.
Trades incur proportional costs: when the position changes by $\Delta_t-\Delta_{t^-}$ at price $S_t$, the cost is
$c\,|\Delta_t-\Delta_{t^-}|\,S_t$, where $c>0$ is the half-spread (or proportional fee) rate.
The portfolio is initialized with the option premium and rebalanced until maturity $T$.
Let $(S_T-K)^+$ denote the terminal payoff of the short call.

\paragraph{Net outcome, execution cost, and replication component.}
Let $W_T$ denote the terminal portfolio value after all transaction costs have been debited through the hedging horizon.
The realized after-cost hedging outcome is
\begin{equation}
\mathrm{PnL}_T^{\mathrm{net}} \;:=\; W_T - (S_T-K)^+ .
\label{eq:pi_net_def}
\end{equation}
We also track the cumulative transaction cost
\begin{equation}
\mathrm{TC}_T \;:=\; \sum_{t\in\mathcal{T}} c\,|\Delta_t-\Delta_{t^-}|\,S_t \;\ge\; 0 ,
\label{eq:ct_def}
\end{equation}
where $\mathcal{T}$ includes the initial hedge trade at $t_0$ and all subsequent daily rebalancing times.
At maturity, the terminal stock position is marked to market in $W_T$ without an explicit closing trade, so no terminal liquidation cost is charged.
To separate replication quality from execution intensity, we define the before-cost replication component
\begin{equation}
\xi_T \;:=\; \mathrm{PnL}_T^{\mathrm{net}} + \mathrm{TC}_T ,
\label{eq:eps_def}
\end{equation}
so that $\mathrm{PnL}_T^{\mathrm{net}}=\xi_T-\mathrm{TC}_T$ holds by construction.
This decomposition is useful because models may differ both in their replication component (pre-cost) and in the trading they induce (cost).

\paragraph{Outcome summaries.}
We report three complementary summaries, each answering a distinct desk-relevant question.

\begin{enumerate}
\item \textit{Distribution of after-cost outcomes (CDF of $\mathrm{PnL}_T^{\mathrm{net}}$).}
We visualize the full distribution of $\mathrm{PnL}_T^{\mathrm{net}}$ using overlaid empirical cumulative distribution functions (ECDFs).
A right-shifted CDF corresponds to better outcomes (larger after-cost P\&L), and crossings reveal regime- or state-dependent trade-offs that are not captured by a single scalar metric.

\item \textit{Downside severity and loss frequency (Expected Shortfall (ES) and shortfall probability).}
We quantify tail losses using the shortfall
\begin{equation}
\mathrm{SF}_T \;:=\; \max\!\bigl(0,\,-\mathrm{PnL}_T^{\mathrm{net}}\bigr),
\label{eq:sf_def}
\end{equation}
and report Expected Shortfall (ES; also known as Conditional Value-at-Risk/CVaR) at tail levels $\alpha\in\{5\%,10\%\}$ \cite{rockafellar2000cvar},
\begin{equation}
\mathrm{ES}_{\alpha} \;:=\; \mathbb{E}\!\left[\mathrm{SF}_T \,\middle|\, \mathrm{SF}_T \ge \mathrm{VaR}_{\alpha}(\mathrm{SF}_T)\right],
\label{eq:es_def}
\end{equation}
where $\mathrm{VaR}_{\alpha}(\mathrm{SF}_T)$ is the $(1-\alpha)$ quantile of $\mathrm{SF}_T$.
Lower values indicate smaller expected shortfalls conditional on being in the worst $\alpha$ tail.
We also report the shortfall probability $\mathbb{P}(\mathrm{PnL}_T^{\mathrm{net}}<0)$ to separate loss frequency from loss severity.

\item \textit{Implementation efficiency (replication dispersion versus cost).}
To summarize the execution--accuracy trade-off, we plot average transaction cost against a dispersion measure of the replication component.
Specifically, we report
\begin{equation}
\mathrm{RMSE}(\xi_T) \;:=\; \sqrt{\mathbb{E}\!\left[\xi_T^2\right]},
\label{eq:rmse_eps_def}
\end{equation}
and place each model as a point in the resulting risk--cost plane $\bigl(\mathbb{E}[\mathrm{TC}_T],\,\mathrm{RMSE}(\xi_T)\bigr)$.
This view is complementary to tail-risk metrics: it isolates whether performance differences are driven primarily by changes in the replication component (pre-cost), changes in trading intensity, or both.
Uncertainty bands are reported as confidence intervals for the mean of $\xi_T^2$, mapped to RMSE by taking square roots of the bounds.
\end{enumerate}

\paragraph{Aggregation across days.}
All statistics are computed with equal-day weighting.
On each trading day and within each maturity bucket, we select the listed contract whose moneyness $K/F$ is closest to each target (ATM and $K/F=1.03$) and run the realized-path hedging backtest for that contract.
We then aggregate the resulting per-day outcomes across days so that each date contributes one observation.

\subsubsection{Distributional Outcomes}
\label{subsubsec:dist_outcomes}

Fig.~\ref{fig:net_cdf_tau28} overlays empirical CDFs of the after-cost net outcome $\mathrm{PnL}_T^{\mathrm{net}}$ (Eq.~\eqref{eq:pi_net_def}) for the $\tau=28$d bucket,
comparing the two moneyness targets ATM ($K/F=1$) and mildly out-of-the-money ($K/F=1.03$) across assets (SPY, XOP) and regimes (2020Q1 versus 2025Q2). Right-shifted curves indicate improved after-cost outcomes, and the vertical reference at $0$ marks the break-even point.

Two qualitative features stand out.
First, the 2020Q1 stress regime exhibits substantially wider dispersion and heavier left tails than 2025Q2, with the clearest separation across models in the
sector ETF (XOP).
Second, curves frequently cross within a slice, indicating that relative performance depends on which part of the outcome distribution is emphasized (middle
versus tail) rather than admitting a single uniform ranking.
Notably, in stress/sector slices where adverse outcomes are most dispersed, the RL policies often display a more favorable left-tail profile after costs; we make this explicit using tail shortfall summaries and an execution-efficiency decomposition reported next.

\begin{figure*}[t]
  \centering
  \includegraphics[width=\textwidth]{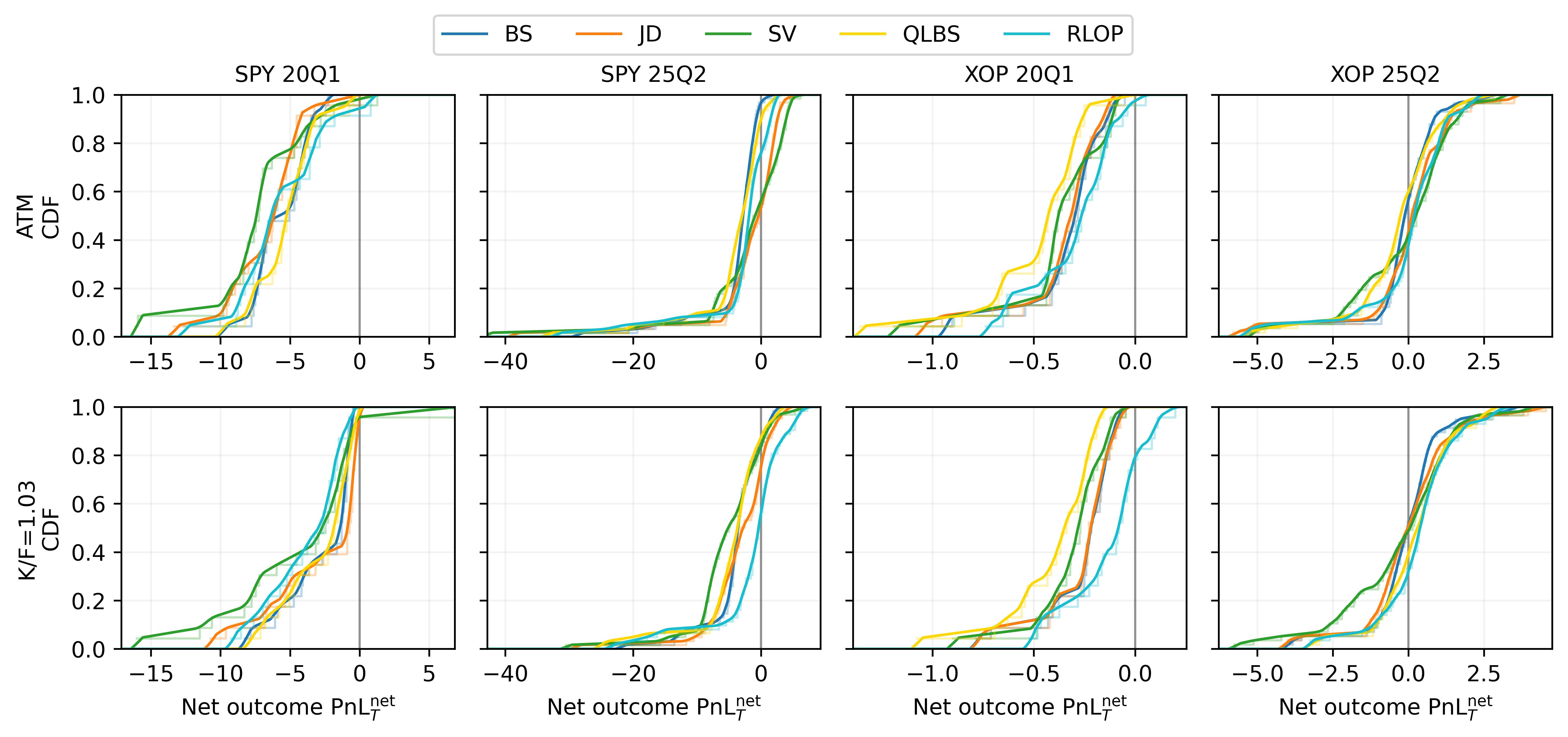}
  \caption{Empirical CDFs of after-cost net hedging outcome $\mathrm{PnL}_T^{\mathrm{net}}$ for $\tau=28$d.
  Columns correspond to SPY 2020Q1, SPY 2025Q2, XOP 2020Q1, and XOP 2025Q2.
  Top row: ATM ($K/F=1$). Bottom row: mildly out-of-the-money ($K/F=1.03$).
  Right-shifted curves indicate improved after-cost outcomes; crossings motivate the explicit tail-risk summaries reported in
  Sec.~\ref{subsubsec:tail_es}.}
  \label{fig:net_cdf_tau28}
\end{figure*}

\subsubsection{Tail Shortfall and Expected Shortfall}
\label{subsubsec:tail_es}

The CDFs in Fig.~\ref{fig:net_cdf_tau28} motivate tail summaries that separate loss frequency from tail-loss severity.
Table~\ref{tab:scorecard_28d_tail} reports, for each $\tau=28$d slice, the model with lowest $\mathrm{ES}_{5\%}$, lowest $\mathrm{ES}_{10\%}$, and lowest shortfall probability $\mathbb{P}(\mathrm{PnL}_T^{\mathrm{net}}<0)$. Two messages stand out.

First, RLOP is the most consistent winner on loss frequency: across the eight slices, it achieves the lowest shortfall probability in six cases (including all four XOP slices). This indicates fewer after-cost losing realizations under the same daily rebalancing schedule, which is operationally meaningful when hedges are deployed repeatedly.

Second, tail-loss severity is more regime-dependent, but RL is strongest in the sector stress slice. Across the eight slices, an RL method (QLBS or RLOP) attains the lowest $\mathrm{ES}_{5\%}$ in five cases, with the clearest improvements concentrated in the stress/sector regime (XOP 2020Q1), where RLOP is best on both $\mathrm{ES}_{5\%}$ and $\mathrm{ES}_{10\%}$ for both moneyness targets. Outside stress/sector conditions, rankings are mixed. Parametric baselines can lead Expected Shortfall even when an RL policy reduces loss frequency, reinforcing that frequency and severity need not align.

We therefore interpret the tail scorecard together with the risk--cost maps that follow, which separate replication dispersion from execution cost.

\begin{table*}[t]
\centering
\small
\setlength{\tabcolsep}{5pt}
\renewcommand{\arraystretch}{1.15}

\begin{tabular}{l c c c c}
\hline
Setting ($\tau=28$) & Best $\mathrm{ES}_{5\%}$ & Best $\mathrm{ES}_{10\%}$ & Lowest shortfall prob. & $n_{\text{days}}$ \\
\hline
SPY 2020Q1 ATM           & BS (8.787)    & BS (8.395)    & RLOP (0.91) & 23 \\
SPY 2020Q1 $K/F{=}1.03$  & QLBS (7.538)  & QLBS (7.271)  & SV (0.96)   & 23 \\
SPY 2025Q2 ATM           & BS (19.214)   & JD (13.995)   & JD (0.53)   & 62 \\
SPY 2025Q2 $K/F{=}1.03$  & JD (15.114)   & JD (12.207)   & RLOP (0.55) & 62 \\
XOP 2020Q1 ATM           & RLOP (0.694)  & RLOP (0.674)  & RLOP (0.96) & 23 \\
XOP 2020Q1 $K/F{=}1.03$  & RLOP (0.502)  & RLOP (0.489)  & RLOP (0.78) & 23 \\
XOP 2025Q2 ATM           & QLBS (4.610)  & BS (2.983)    & RLOP (0.39) & 57 \\
XOP 2025Q2 $K/F{=}1.03$  & QLBS (3.037)  & QLBS (2.075)  & RLOP (0.33) & 57 \\
\hline
\end{tabular}

\caption{Tail scorecard for $\tau=28$d bucket: lowest $\mathrm{ES}_{5\%}$, $\mathrm{ES}_{10\%}$, and shortfall probability $\mathbb{P}(\mathrm{PnL}_T^{\mathrm{net}}<0)$ under equal-day weighting; $n_{\text{days}}$ is the number of trading days.}
\label{tab:scorecard_28d_tail}
\end{table*}

\subsubsection{Risk--Cost Map: Replication Risk versus Execution Cost}
\label{subsubsec:risk_cost_map}

The CDF and tail summaries focus on after-cost outcomes, but it is also useful to separate execution intensity from replication dispersion.
Fig.~\ref{fig:risk_cost_tau28} summarizes this decomposition for the $\tau=28$d bucket by placing each model at
$\bigl(\mathbb{E}[\mathrm{TC}_T],\,\mathrm{RMSE}(\xi_T)\bigr)$, where $\mathrm{TC}_T$ is cumulative transaction cost (Eq.~\eqref{eq:ct_def}),
$\xi_T=\mathrm{PnL}_T^{\mathrm{net}}+\mathrm{TC}_T$ is the before-cost replication component (Eq.~\eqref{eq:eps_def}),
and $\mathrm{RMSE}(\xi_T)$ is defined in Eq.~\eqref{eq:rmse_eps_def}.
Expectations are equal-day means over trading days for the target contract (closest listed $K/F$ each day).
Lower-left is better (cheaper hedging with smaller replication dispersion).

Two desk-relevant patterns are visible.
First, the lowest average trading cost is consistently achieved by an RL policy (QLBS or RLOP) across these slices, indicating systematically lower turnover under the same daily rebalancing schedule.
Second, the risk--cost trade-off depends on the slice: in several settings the RL points lie on or near the lower envelope (delivering low dispersion at low cost), while in others a parametric benchmark can attain slightly lower dispersion only by paying materially higher cost.
This view complements Sec.~\ref{subsubsec:tail_es} by clarifying whether differences in tail outcomes are driven primarily by reduced execution, improved replication, or both.

\begin{figure*}[t]
\centering
\includegraphics[width=\textwidth]{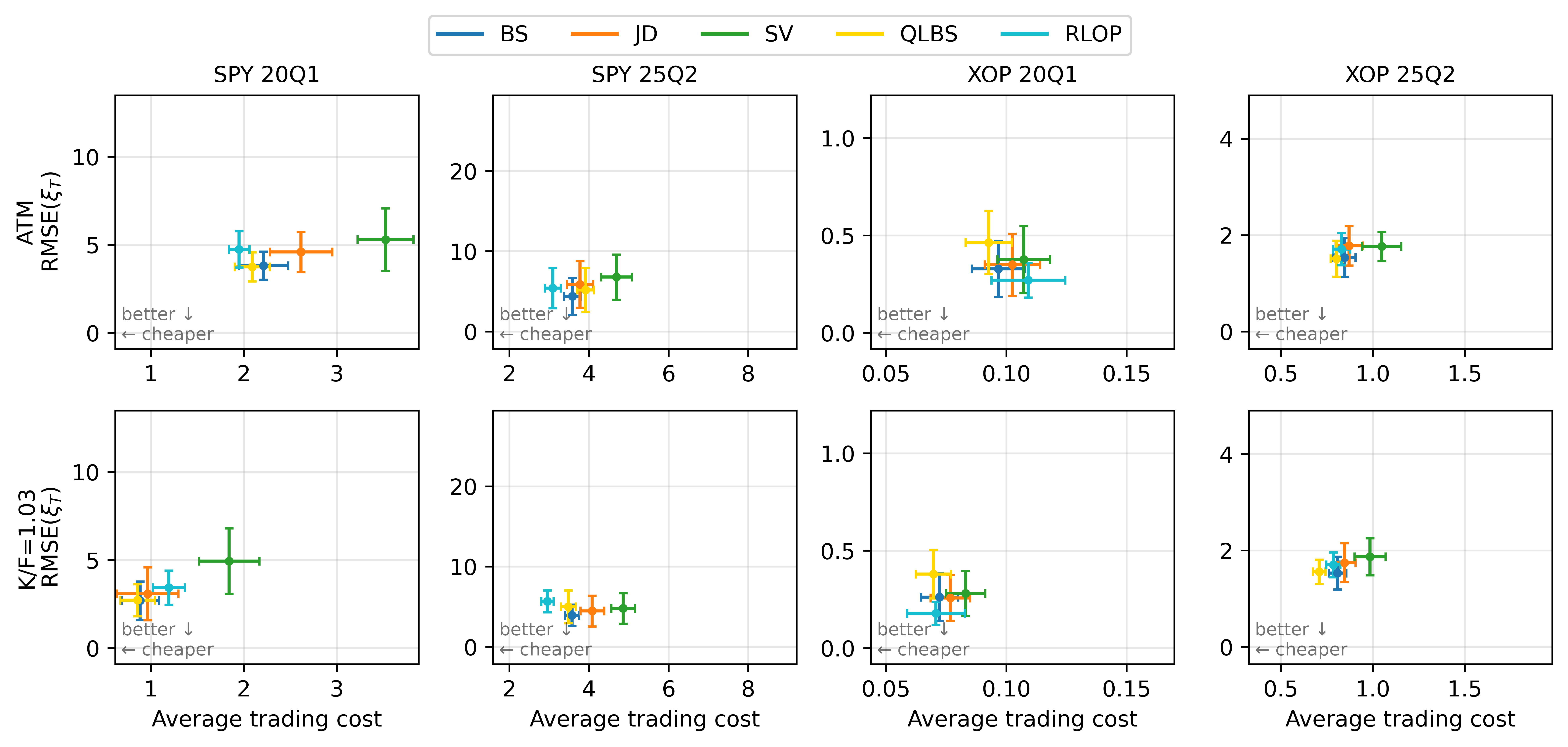}
\caption{Risk--cost maps for the $\tau=28$d bucket.
Each point plots average transaction cost $\mathbb{E}[\mathrm{TC}_T]$ versus replication dispersion $\mathrm{RMSE}(\xi_T)$, with $\xi_T=\mathrm{PnL}_T^{\mathrm{net}}+\mathrm{TC}_T$.
Error bars are 95\% confidence intervals; lower-left indicates cheaper hedging with lower replication dispersion.}
\label{fig:risk_cost_tau28}
\end{figure*}

\subsection{Static Pricing Accuracy: Implied-Volatility Fit}
\label{subsec:ivrmse}

As a diagnostic for same-day cross-sectional pricing accuracy, we report the equal-day implied-volatility RMSE (IVRMSE) for the $\tau=28$d maturity bucket.
On each trading day $d$, we invert Black--76 \cite{black1976commodity} to compute implied volatilities from both market prices and each model’s fitted prices, and define the per-day statistic
\begin{equation}
\mathrm{IVRMSE}^{(m)}_{d}
\;:=\;
10^{2}\times
\sqrt{\frac{1}{N_d}\sum_{i=1}^{N_d}\Bigl(\sigma^{(m)}_{d,i}-\sigma^{\mathrm{mkt}}_{d,i}\Bigr)^2},
\label{eq:ivrmse_def}
\end{equation}
where the sum runs over the $N_d$ contracts in the relevant day--bucket slice and $m\in\{\mathrm{BS},\mathrm{JD},\mathrm{SV},\mathrm{QLBS},\mathrm{RLOP}\}$.
We then aggregate across days using equal-day weighting, i.e., for each model $m$ we report
$\frac{1}{|\mathcal{D}|}\sum_{d\in\mathcal{D}}\mathrm{IVRMSE}^{(m)}_{d}$ so that each date contributes one observation, independent of the number of listed strikes.
Lower values indicate closer agreement with the observed implied-volatility surface on the calibration day.

Table~\ref{tab:ivrmse_28d_by_moneyness} reports IVRMSE for the full $\tau=28$d slice and for broad moneyness subsets.
We emphasize that IVRMSE is reported to contextualize pricing misspecification and model mismatch, rather than to rank hedging strategies.

\begin{table*}[t]
\centering
\small
\begin{tabular}{cc c c c c c c}
\toprule
Moneyness, $\tau$ & Period & Asset & BS & JD & SV & QLBS & RLOP \\
\midrule
\multirow{4}{*}{Whole sample, 28d}
& \multirow{2}{*}{2020Q1} & SPY & 7.65 & \textbf{1.76} & 4.21 & 10.27 & 8.25 \\
&                        & XOP & 12.02 & 9.16 & \textbf{8.83} & 10.99 & 18.81 \\
\cmidrule(lr){2-8}
& \multirow{2}{*}{2025Q2} & SPY & 12.80 & 9.49 & 7.40 & 9.25 & \textbf{7.34} \\
&                        & XOP & 10.61 & \textbf{6.48} & 7.33 & 11.12 & 16.36 \\
\midrule
\multirow{4}{*}{Moneyness $<1$, 28d}
& \multirow{2}{*}{2020Q1} & SPY & 8.84 & \textbf{1.61} & 4.13 & 13.39 & 9.53 \\
&                        & XOP & 14.09 & \textbf{9.80} & 10.22 & 12.48 & 19.52 \\
\cmidrule(lr){2-8}
& \multirow{2}{*}{2025Q2} & SPY & 16.94 & 11.42 & 7.55 & 10.76 & \textbf{7.55} \\
&                        & XOP & 12.37 & \textbf{7.51} & 8.61 & 13.46 & 13.73 \\
\midrule
\multirow{4}{*}{Moneyness $>1$, 28d}
& \multirow{2}{*}{2020Q1} & SPY & 5.46 & \textbf{1.75} & 3.21 & 5.46 & 6.37 \\
&                        & XOP & 7.64 & 5.35 & \textbf{4.58} & 6.60 & 16.47 \\
\cmidrule(lr){2-8}
& \multirow{2}{*}{2025Q2} & SPY & \textbf{3.97} & 5.36 & 5.93 & 6.52 & 6.36 \\
&                        & XOP & 7.94 & \textbf{4.33} & 4.95 & 6.60 & 18.23 \\
\midrule
\multirow{4}{*}{Moneyness $>1.03$, 28d}
& \multirow{2}{*}{2020Q1} & SPY & 6.19 & \textbf{2.02} & 3.70 & 4.17 & 5.59 \\
&                        & XOP & 8.14 & 5.64 & \textbf{4.92} & 6.99 & 16.67 \\
\cmidrule(lr){2-8}
& \multirow{2}{*}{2025Q2} & SPY & \textbf{4.04} & 5.94 & 6.58 & 6.76 & 6.12 \\
&                        & XOP & 8.54 & \textbf{4.61} & 5.29 & 6.69 & 19.14 \\
\bottomrule
\end{tabular}
\caption{Equal-day IVRMSE (Eq.~\eqref{eq:ivrmse_def}) for the $\tau=28$d maturity bucket across moneyness groups.
Lower is better; bold marks the best value within each asset, period, and moneyness row.}
\label{tab:ivrmse_28d_by_moneyness}
\end{table*}

Two qualitative patterns are evident.
First, the parametric benchmarks typically dominate static surface fit across assets and regimes: a parametric model is best in the majority of rows, with JD
often leading and SV frequently competitive.
This is most pronounced in the 2020Q1 stress regime, where jumps and stochastic volatility materially improve cross-sectional fit relative to the constant-volatility
BS benchmark.

Second, the RL policies are not designed to minimize same-day surface error and therefore do not consistently lead IVRMSE.
This is expected: their objective is to produce deltas that perform well under realized-path hedging with transaction costs, where trading feedback and
path dependence matter.
Notably, in calmer index conditions (SPY 2025Q2), RLOP becomes competitive on IVRMSE in parts of the slice, indicating that a friction-aware hedging policy can
occasionally align with the static surface despite not being trained as an implied-volatility interpolator.

Overall, Table~\ref{tab:ivrmse_28d_by_moneyness} supports the interpretation used throughout this section:
strong same-day IV fit is attainable with standard parametric models, but it is not a reliable proxy for realized-path hedging performance under transaction costs.
We therefore treat IVRMSE as a supporting diagnostic and place primary emphasis on distributional outcomes, tail risk, and implementation efficiency in
Sec.~\ref{subsec:dynamic_hedging}.

\subsection{Robustness and Implications}
\label{subsec:robustness_implications}

The main text focuses on the $\tau=28$d bucket to keep the empirical narrative compact.
Figures \ref{fig:net_cdf_tau14}--\ref{fig:risk_cost_tau56} in the appendix replicate the net-CDF overlays and the risk--cost maps for $\tau\in\{14,56\}$, and Tables \ref{tab:scorecard_14d}--\ref{tab:ivrmse_56d_whole} in the appendix report the corresponding
tail-risk scorecards and IVRMSE breakdowns.
Across horizons, the qualitative conclusions are stable.
Static IV fit continues to favor the parametric benchmarks, while realized-path hedging under transaction costs reveals differences in execution spend,
after-cost outcome distributions, and downside risk.

These results also clarify why static fit and realized hedging performance need not align.
IVRMSE is a same-day cross-sectional diagnostic evaluated on the calibration surface, whereas hedging outcomes are determined by the subsequently realized
underlying path under discrete rebalancing, together with the interaction between turnover and proportional costs.
As a result, a model can interpolate the surface well and still generate deltas that are suboptimal once execution frictions and realized dynamics are
accounted for.

Taken together, the evidence supports a hedging-facing evaluation criterion under costs.
The net-CDF overlays localize where after-cost gains occur and where rankings reverse due to crossings.
Expected Shortfall summarizes extreme-loss severity, while the risk--cost maps separate replication dispersion (via $\xi_T$) from execution intensity
(via $\mathrm{TC}_T$).
Across slices, the RL policies deliver systematic reductions in trading cost, and the tail-risk benefits are most visible in stress and sector conditions.
The two RL methods also retain economic interpretability: QLBS is more replication-oriented, whereas RLOP more consistently prioritizes implementability and
downside control under frictions.

From a model-risk perspective, the results indicate that transaction-cost-aware policies can improve realized hedging outcomes even when they do not dominate
static surface fit.
This distinction is operationally relevant for trading and risk functions that hedge through regime shifts and misspecification, and it motivates the
horizon-robustness results reported in the appendix.
In applied terms, reducing turnover while controlling tail losses constitutes a market-facing benefit, supporting the view that learning-based policies can
complement classical parametric models by optimizing the deployment objective directly.
\section{Conclusion}

This paper develops modified QLBS and novel RLOP frameworks for autonomous option hedging, addressing the systemic risks inherent in the decoupling of model calibration from execution performance. The empirical cost-risk maps demonstrate a systematic cost advantage, with reinforcement learning consistently minimizing turnover costs across all tested regimes. While diagnostics like IVRMSE favor parametric models, we demonstrate that static fit is an insufficient proxy for hedging quality under market frictions. Our full-distribution analysis reveals that RL policies materially mitigate tail-risk losses, particularly during operationally salient periods like the 2020 liquidity crash. RLOP emerges as a critical tool for capital-constrained desks by prioritizing survival-centric hedging over loss magnitude. Future research will explore computational complexity, path-dependent instruments, funding-spread jumps, and model risk to further assess the macroeconomic implications of AI-augmented risk management on market equilibrium.

\bibliographystyle{elsarticle-num}
\bibliography{ref}

\onecolumn
\appendix
\section{Additional Market-Data Results: Horizon Robustness}
\label{app:market_horizon}

This appendix reports the $\tau\in\{14,56\}$ maturity-bucket counterparts of the main-text $\tau=28$d analysis.
Figures~\ref{fig:net_cdf_tau14}--\ref{fig:risk_cost_tau56} replicate the distributional and efficiency views, while
Tables~\ref{tab:scorecard_14d}--\ref{tab:ivrmse_56d_whole} provide tail-risk and static-fit summaries.
All definitions match Section~\ref{sec:empirical_market} unless stated otherwise.

\begin{figure}[H]
  \centering
  \includegraphics[width=\textwidth]{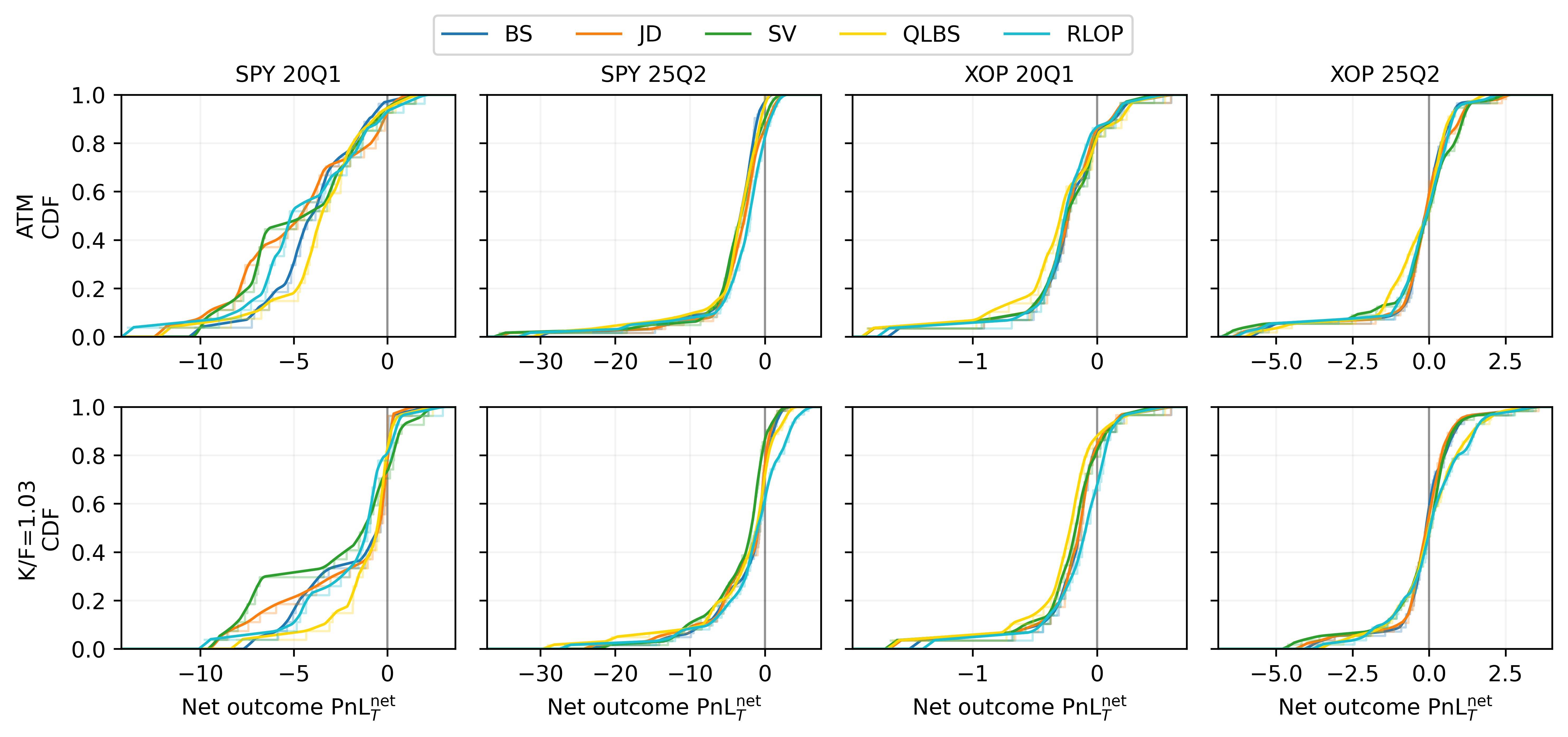}
  \caption{Empirical CDFs of after-cost net hedging outcome $\mathrm{PnL}_T^{\mathrm{net}}$ for $\tau=14$d.
  Columns correspond to SPY 2020Q1, SPY 2025Q2, XOP 2020Q1, and XOP 2025Q2.
  Top row: ATM ($K/F=1$). Bottom row: mildly out-of-the-money ($K/F=1.03$).
  Right-shifted curves indicate improved after-cost outcomes.}
  \label{fig:net_cdf_tau14}
\end{figure}

\begin{figure}[H]
  \centering
  \includegraphics[width=\textwidth]{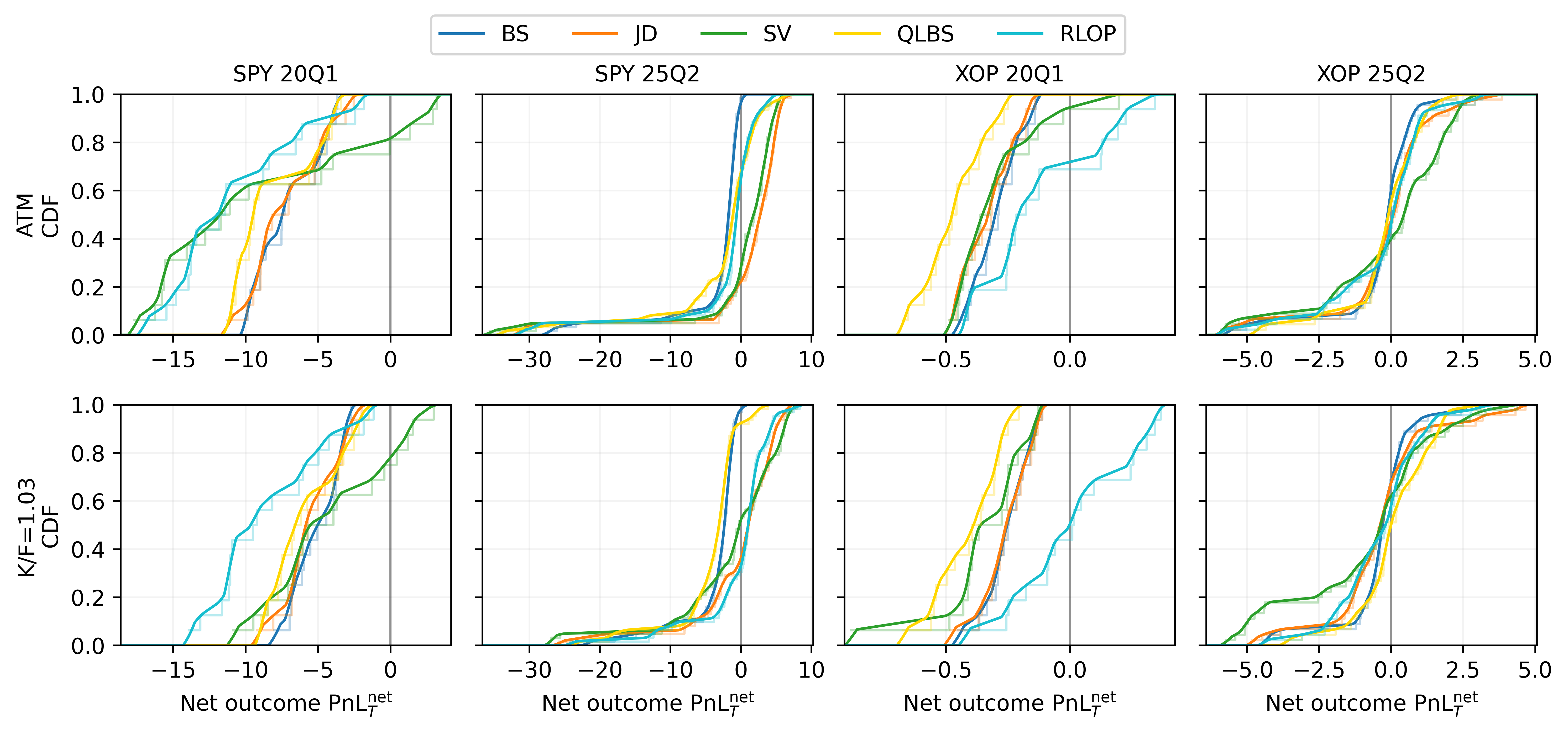}
  \caption{Empirical CDF overlays of after-cost net hedging P\&L $\mathrm{PnL}_T^{\mathrm{net}}$ for the $\tau=56$d bucket.
  Layout and interpretation match Figure~\ref{fig:net_cdf_tau14}.}
  \label{fig:net_cdf_tau56}
\end{figure}

\begin{table}[H]
\centering
\small
\setlength{\tabcolsep}{5pt}
\renewcommand{\arraystretch}{1.15}
\begin{tabular}{l c c c c}
\hline
Setting ($\tau=14$) & Best $\mathrm{ES}_{5\%}$ & Best $\mathrm{ES}_{10\%}$ & Lowest shortfall prob. & $n_{\text{days}}$ \\
\hline
SPY 2020Q1 ATM          & BS (8.757)    & BS (8.210)    & JD/SV/QLBS/RLOP (0.93) & 27 \\
SPY 2020Q1 $K/F{=}1.03$ & QLBS (6.187)  & QLBS (5.154)  & SV (0.70)              & 27 \\
SPY 2025Q2 ATM          & BS (17.711)   & JD (13.385)   & JD (0.84)              & 62 \\
SPY 2025Q2 $K/F{=}1.03$ & SV (14.535)   & BS (12.200)   & RLOP (0.60)            & 62 \\
XOP 2020Q1 ATM          & RLOP (1.201)  & RLOP (0.987)  & BS/SV/QLBS (0.83)      & 29 \\
XOP 2020Q1 $K/F{=}1.03$ & RLOP (0.935)  & RLOP (0.769)  & RLOP (0.66)            & 29 \\
XOP 2025Q2 ATM          & QLBS (5.096)  & BS (3.340)    & SV/QLBS (0.52)         & 56 \\
XOP 2025Q2 $K/F{=}1.03$ & RLOP (2.711)  & RLOP (2.210)  & SV (0.46)              & 56 \\
\hline
\end{tabular}
\caption{Tail scorecard for $\tau=14$d bucket: lowest $\mathrm{ES}_{5\%}$, $\mathrm{ES}_{10\%}$, and shortfall probability $\mathbb{P}(\mathrm{PnL}_T^{\mathrm{net}}<0)$ under equal-day weighting; $n_{\text{days}}$ is the number of trading days.}
\label{tab:scorecard_14d}
\end{table}

\begin{table}[H]
\centering
\small
\setlength{\tabcolsep}{5pt}
\renewcommand{\arraystretch}{1.15}
\begin{tabular}{l c c c c}
\hline
Setting ($\tau=56$) & Best $\mathrm{ES}_{5\%}$ & Best $\mathrm{ES}_{10\%}$ & Lowest shortfall prob. & $n_{\text{days}}$ \\
\hline
SPY 2020Q1 ATM          & BS (9.963)    & BS (9.900)    & SV (0.81)   & 16 \\
SPY 2020Q1 $K/F{=}1.03$ & BS (8.002)    & BS (7.478)    & SV (0.75)   & 16 \\
SPY 2025Q2 ATM          & BS (22.015)   & JD (14.806)   & JD (0.23)   & 62 \\
SPY 2025Q2 $K/F{=}1.03$ & RLOP (15.437) & RLOP (12.079) & RLOP (0.31) & 62 \\
XOP 2020Q1 ATM          & RLOP (0.424)  & RLOP (0.417)  & RLOP (0.69) & 16 \\
XOP 2020Q1 $K/F{=}1.03$ & RLOP (0.424)  & RLOP (0.340)  & RLOP (0.44) & 16 \\
XOP 2025Q2 ATM          & QLBS (3.898)  & QLBS (3.157)  & SV (0.40)   & 45 \\
XOP 2025Q2 $K/F{=}1.03$ & QLBS (2.893)  & QLBS (2.293)  & QLBS (0.51) & 45 \\
\hline
\end{tabular}
\caption{Tail scorecard for $\tau=56$d bucket: lowest $\mathrm{ES}_{5\%}$, $\mathrm{ES}_{10\%}$, and shortfall probability $\mathbb{P}(\mathrm{PnL}_T^{\mathrm{net}}<0)$ under equal-day weighting; $n_{\text{days}}$ is the number of trading days.}
\label{tab:scorecard_56d}
\end{table}

\begin{figure}[H]
  \centering
  \includegraphics[width=\textwidth]{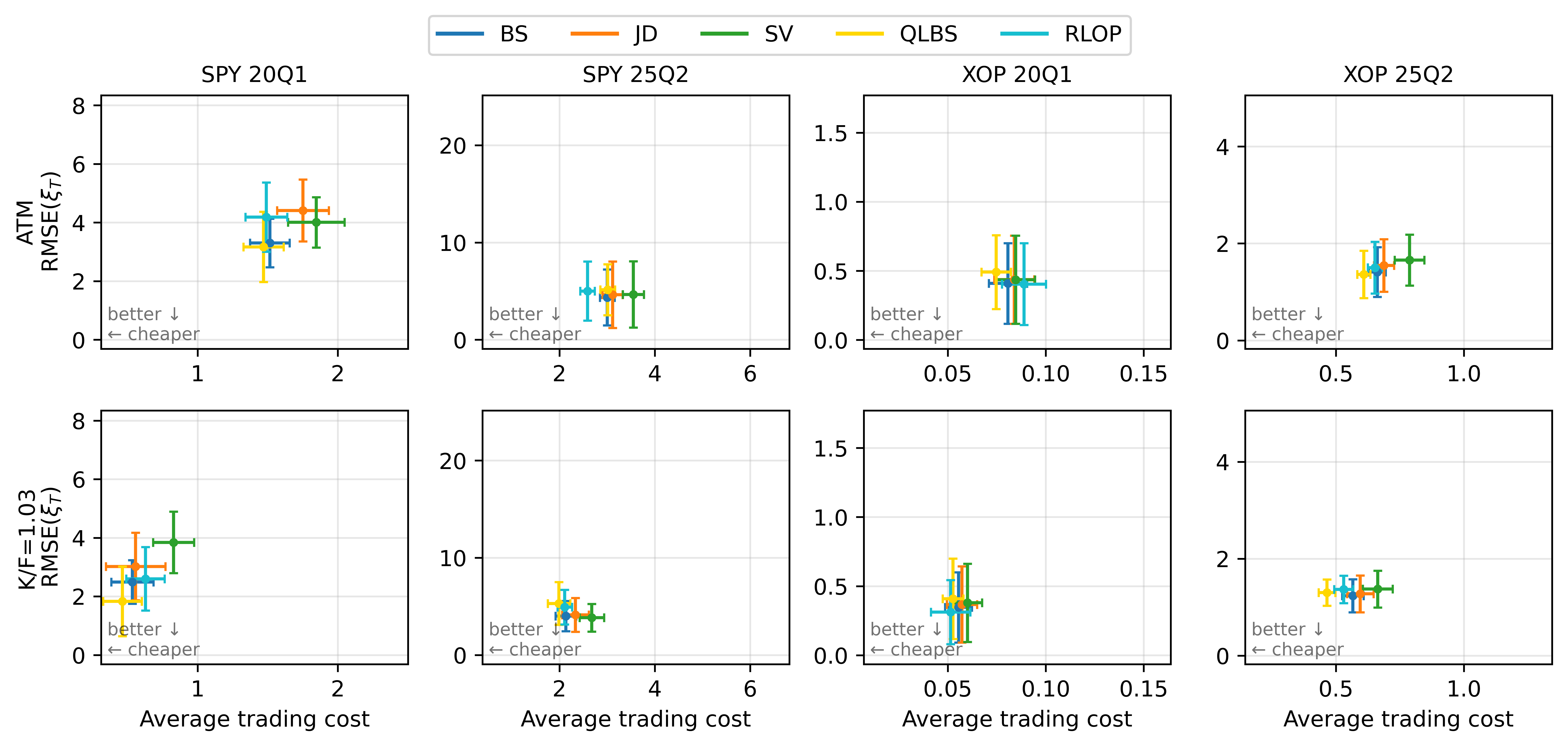}
  \caption{Risk--cost maps for the $\tau=14$d bucket.
Each point plots average transaction cost $\mathbb{E}[\mathrm{TC}_T]$ versus replication dispersion $\mathrm{RMSE}(\xi_T)$, with $\xi_T=\mathrm{PnL}_T^{\mathrm{net}}+\mathrm{TC}_T$.
Error bars are 95\% confidence intervals; lower-left indicates cheaper hedging with lower replication dispersion}
  \label{fig:risk_cost_tau14}
\end{figure}

\begin{figure}[H]
  \centering
  \includegraphics[width=\textwidth]{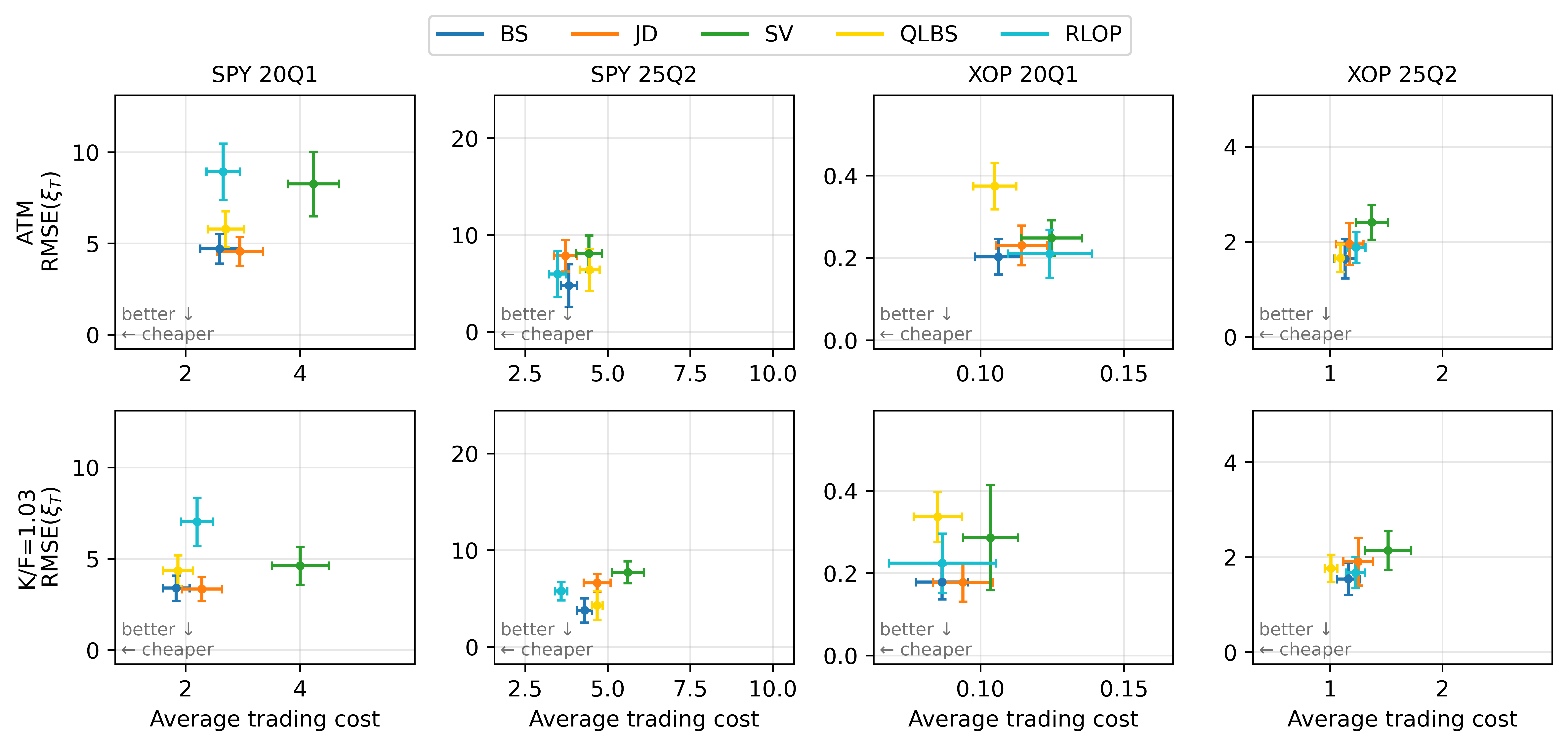}
  \caption{Risk--cost maps for the $\tau=56$d bucket. Layout and interpretation match Figure~\ref{fig:risk_cost_tau14}.}
  \label{fig:risk_cost_tau56}
\end{figure}

\begin{table}[H]
\centering
\begin{tabular}{cc c c c c c c}
\toprule
Moneyness, $\tau$ & Period & Asset & BS & JD & SV & QLBS & RLOP \\
\midrule
\multirow{4}{*}{Whole sample, 14d}
& \multirow{2}{*}{2020Q1} & SPY & 9.15  & \textbf{5.53} & 7.47  & 12.62 & 10.74 \\
&                        & XOP & 18.77 & \textbf{16.39} & 18.87 & 20.39 & 24.94 \\
\cmidrule(lr){2-8}
& \multirow{2}{*}{2025Q2} & SPY & 16.44 & 13.60 & 10.81 & 11.83 & \textbf{9.49} \\
&                        & XOP & 15.17 & \textbf{9.56} & 12.31 & 15.16 & 20.10 \\
\bottomrule
\end{tabular}
\caption{Equal-day IVRMSE for the whole-slice $\tau=14$d bucket. Lower is better; bold denotes the row-wise best within each asset and period.}
\label{tab:ivrmse_14d_whole}
\end{table}

\begin{table}[H]
\centering
\begin{tabular}{cc c c c c c c}
\toprule
Moneyness, $\tau$ & Period & Asset & BS & JD & SV & QLBS & RLOP \\
\midrule
\multirow{4}{*}{Whole sample, 56d}
& \multirow{2}{*}{2020Q1} & SPY & 6.55  & \textbf{1.23} & 2.27  & 8.73  & 8.41 \\
&                        & XOP & 5.94  & 3.34 & \textbf{3.18} & 6.01  & 19.91 \\
\cmidrule(lr){2-8}
& \multirow{2}{*}{2025Q2} & SPY & 9.27  & 6.06 & \textbf{4.36} & 7.44  & 7.05 \\
&                        & XOP & 6.88  & 3.79 & \textbf{3.76} & 7.21  & 14.74 \\
\bottomrule
\end{tabular}
\caption{Equal-day IVRMSE for the whole-slice $\tau=56$d bucket. Lower is better; bold denotes the row-wise best within each asset and period.}
\label{tab:ivrmse_56d_whole}
\end{table}


\end{document}